\documentclass[15pt]{article}

\usepackage[numbers]{natbib}

\usepackage{microtype}
\usepackage{graphicx}
\usepackage{subfigure}
\usepackage{booktabs}
\usepackage{amsmath}
\usepackage{algorithm}
\usepackage{algorithmic}
\usepackage{rotating}
\usepackage{changepage}
\usepackage[thmmarks,amsmath]{ntheorem}
\usepackage{multirow}
\usepackage{amsfonts,amssymb}
\usepackage[T1]{fontenc}
\usepackage{url}
\usepackage{epstopdf}

\newcommand{\X}{{\mathcal X}}
\newcommand{\Y}{{\mathcal Y}}

\newcommand{\R}{{\mathbb R}}

\newcommand{\be}{\begin{eqnarray}}
\newcommand{\ben}{\begin{eqnarray*}}
\newcommand{\en}{\end{eqnarray}}
\newcommand{\enn}{\end{eqnarray*}}

\newcommand{\half}{\frac{1}{2}}

\newtheorem{theorem}{Theorem}

\newenvironment{proof}{{\em Proof.}}{$\Box$}

\usepackage{lineno,hyperref}

\title{Learning with Smooth Hinge Losses}
\author{Junru~Luo 
\thanks{J. Luo is with the School of Computer Science and Artificial Intelligence $\&$ Aliyun School of Big Data,
Changzhou University, Changzhou, Jiangsu province, China e-mail: (luojunru@cczu.edu.cn). } 
        , Hong~Qiao
\thanks{H. Qiao is the Institute of Automation, Chinese Academy of Sciences, Beijing 100190, China
and School of Artificial Intelligence, University of Chinese Academy of Sciences, Beijing 100049, China
e-mail: (hong.qiao@ia.ac.cn).}
        , and~Bo~Zhang
\thanks{B. Zhang is the Academy of Mathematics and Systems Science, Chinese Academy of Sciences, Beijing 100190, China
and School of Mathematical Sciences, University of Chinese Academy of Sciences, Beijing 100049, China
e-mail: (b.zhang@amt.ac.cn).}}

\begin{document}
\bibliographystyle{plain}
\maketitle

\begin{abstract}
Due to the non-smoothness of the Hinge loss in SVM, it is difficult to obtain a faster convergence rate with
modern optimization algorithms.
In this paper, we introduce two smooth Hinge losses $\psi_G(\alpha;\sigma)$ and $\psi_M(\alpha;\sigma)$
which are infinitely differentiable and converge to the Hinge loss uniformly in $\alpha$ as $\sigma$ tends to $0$.
By replacing the Hinge loss with these two smooth Hinge losses, we obtain two smooth support vector machines
(SSVMs), respectively. Solving the SSVMs with the Trust Region Newton method (TRON) leads to two quadratically
convergent algorithms. Experiments in text classification tasks show that the proposed SSVMs are effective in
real-world applications.
We also introduce a general smooth convex loss function to unify several commonly-used convex loss functions
in machine learning. The general framework provides smooth approximation functions to non-smooth convex
loss functions, which can be used to obtain smooth models that can be solved with faster convergent
optimization algorithms.
\end{abstract}




\section{Introduction}
\label{sec1}

Consider binary classification problems.
Suppose we have a training dataset $\{(x_1,y_1),(x_2,y_2),\cdots,(x_n,y_n)\}$, which is assumed to be independent
and identically distributed realizations of a random pair $\X\times\Y$, where $\X\subset\R^p$ and $\Y =\{+1,-1\}$.
Our purpose is to learn a linear classifier $w\in\R^p$ to predict a new instance $x\in\X$ correctly.
The instance $x$ will be assigned to be positive if $w^Tx>0$, and negative otherwise.
Moreover, we use the $0/1$ loss to evaluate the performance of the classifier $w$, that is,
if the classifier $w$ makes a correct decision, then there is no loss and, otherwise, the loss is $1$.
To avoid overfitting, it is necessary to apply a regularization term to penalize the classifier.
$L_2$-regularization is the mostly used one in machine learning problems, which
allows for the use of a kernel function as a way of embedding the original data in a higher
dimension space \cite{Blanco2020,Shalev2011}.
In certain practical applications such as text classification, one hopes to learn a sparse classifier
through $L_1$-regularization \citep{Hong2019,Chen2019}.
By minimizing the structural risk, the binary classification problem is equivalent to the optimization problem:
\begin{align*}\label{binary-model}
\arg\min_{w\in\R^p}\half\lambda\|w\|^2 + \frac1n\sum_{i=1}^n\mathbb{I}(y_iw^Tx_i\le 0),
\end{align*}
where $\mathbb I(\cdot)$ returns $1$ if its argument is true and $0$ otherwise, and $\half\lambda\|w\|^2 $ is
the $L_2$-regularization term.
Due to the non-differentiability and non-convexity of the $0/1$ loss, it is an NP-hard problem to
optimize \eqref{binary-model} directly. To overcome this difficulty, it is common to replace the $0/1$ loss
with a convex surrogate loss function.
And many efficient convex optimization methods can thus be applied to obtain a good solution,
such as the gradient-based method and the coordinate descent method \cite{Shalev2011,Zheng2013,Chang2008,Zhang2004},
which are iteration methods based on taking the gradient or coordinate as a descent direction
to decrease the objective function.

What kind of convex functions can be applied to replace the $0/1$ loss has been studied
\cite{Lugosi2004,Steinwart2005,Zhang2004b,Bartlett2006}.
The weakest possible condition on the surrogate loss $\ell$ is that
it is {\em classification-calibrated}, which is a pointwise form of the Fisher consistency
for binary classification \cite{Lin2004}.
\cite{Bartlett2006} obtained a necessary and sufficient condition for a convex surrogate loss $\ell$ to be
classification-calibrated, as stated in the following theorem.

\begin{theorem}\label{Bartlett}[\cite{Bartlett2006}, Theorem 2]
A convex function $\ell$ is classification-calibrated if and only if it is differentiable at $0$ and
$\ell'(0)<0.$
\end{theorem}

A brief overview of surrogate loss functions frequently used in practice is given in \cite{Lin2004,Teo2007}.
The mostly used surrogate loss functions include the Hinge loss for SVM, the logistic loss for logistical regression,
and the exponential loss for AdaBoost. They are all classification-calibrated.
Note that the logistic and exponential losses are smooth, but the Hinge loss is not.
As a result, solving SVMs with gradient-based methods only gives a suboptimal convergence rate,
and no second-order algorithm is available for solving SVMs. To address this issue,
the squared Hinge loss was introduced, but this leads to a new model since the squared Hinge loss is not
an approximation of the Hinge loss \cite{Chang2008,Lee2001}.

In this paper, we propose two new convex surrogate losses $\psi_G(\alpha;\sigma)$ and $\psi_{M}(\alpha;\sigma)$
for binary classification, where $\sigma$ is a tunable hyper-parameter (see \eqref{SmoothHL} below),
which are called smooth Hinge losses due to the following reasons.
First, $\psi_G(\alpha;\sigma)$ and $\psi_{M}(\alpha;\sigma)$ converge to the Hinge loss uniformly in $\alpha$
as $\sigma$ approaches to $0$, so they can keep the advantage of the Hinge loss in SVMs.
Secondly, $\psi_G(\alpha;\sigma)$ and $\psi_{M}(\alpha;\sigma)$ are infinitely differentiable.
By replacing the Hinge loss with these two smooth Hinge losses, we obtain two smooth support vector machines (SSVMs)
which can be solved with second-order methods.
In particular, they can be solved by the inexact Newton method with a quadratic convergence rate
as conducted in \cite{Agarwal2017,Lin2008} for the logistic regression.
Although first-order methods are often sufficient in machine learning,
there will be a great improvement in training time experimentally on the large scale sparse learning problems by using second-order methods.

Motivated by the proposed smooth Hinge losses, we also propose a general smooth convex loss function
$\psi(\alpha)=\Phi_c(v)(\theta-\alpha)+\phi_c(v)\sigma$ with $v=(\theta-\alpha)/\sigma$,
where $\Phi_c(v)$ and $\phi_c(v)$ satisfy the conditions given in Theorem \ref{GCL} below.
This general smooth convex loss function $\psi(\alpha)$ provides a smooth approximation to several surrogate loss
functions usually used in machine learning, such as the non-differentiable absolute loss which is usually
used as a regularization term, and the rectified linear unit (ReLU) activation function used in deep neural networks.

This paper is organized as follows. In Section \ref{sec2}, we first briefly review several SVMs with different convex
loss functions and then introduce the smooth Hinge loss functions $\psi_G(\alpha;\sigma),\psi_{M}(\alpha;\sigma)$.
The general smooth convex loss function $\psi(\alpha)$ is then presented and discussed in Section \ref{sec3}.
In Section \ref{sec4}, we give the smooth support vector machine by replacing the Hinge loss with the smooth Hinge
loss $\psi_G$ or $\psi_M$.
The first-order and second-order algorithms for the proposed SSVMs are also presented and analyzed.
Several empirical examples of text categorization with high dimensions and sparse features are implemented
in Section \ref{sec5}; the results show that the smooth Hinge losses are efficient for binary classification.
Some conclusions are given in Section \ref{sec6}.

\section{Smooth Hinge Losses}\label{sec2}

The support vector machine (SVM) is a famous algorithm for binary classification and has now also been applied
to many other machine learning problems such as the AUC learning, multi-task learning, multi-class classification
and imbalanced classification problems \cite{Rakotomamonjy2004,Lima2018,Mei2019,Iranmehr2019}.
\cite{Cervantes2020} is a recent survey work about the applications, challenges and trends of SVM.

The SVM model can be described as the following optimization problem
\be\label{SVM}
\arg\min_{w\in\R^d}\frac12\lambda\|w\|^2 +\frac1n\sum_{i=1}^n \ell(y_iw^Tx_i),
\en
where the used surrogate loss is the Hinge loss$\ell(\alpha)=\max\{0,1-\alpha\}$.
The model \eqref{SVM} is called L1-SVM.

Since the Hinge loss is not smooth, it is usually replaced with a smooth function.
One is the squared Hinge loss $\ell(\alpha)=\max\{0,1-\alpha\}^2$,
which is convex, piecewise quadratic and differentiable \cite{Chang2008,Wangg2020}.
The SVM model with the squared Hinge loss is called L2-SVM.
\cite{Lee2001} proposed to replace the squared Hinge loss by its smooth approximation
$\ell(\alpha)=(1-\alpha)+\sigma\ln(1+e^{-\frac{1-\alpha}{\sigma}})$.
In order to make the objective of the Lagrangian dual problem of L1-SVM strongly convex,
which is needed in developing an accurate optima estimation with dual methods,
the following smoothed hinge loss is proposed in \cite{Shalev-Shwartz2013} to replace the hinge loss:
\ben
\ell_\gamma(\alpha) = \begin{cases}
0  &\text{ if } \alpha \ge 1 \\
1-\alpha-\frac{\gamma}{2} &\text{ if } \alpha \le 1- \gamma \\
\frac{1}{2\gamma}(1-\alpha)^2  &\text{ otherwise.}
\end{cases}
\enn
The stochastic dual coordinate ascent method is then applied to accelerate the training
process \citep{Shalev-Shwartz2013,Shalev2016}.
With the help of $L_1$-regularization and the smoothed hinge-loss $\ell_\gamma(\cdot)$,
a sparse and smooth support vector machine is obtained in \cite{Hong2019}.
By simultaneously identifying the inactive features and samples, a novel screening method
was further developed in \citep{Hong2019}, which is able to reduce a large-scale problem to a small-scale problem.

Motivated by the smoothing technique in quantile regression, \cite{Wang2019} presents the smooth approximation $K_h(\alpha)=(1-\alpha)H((1-\alpha)/h)$ to the hinge loss, where $h$ is a bandwidth and $H(\cdot)$ is the smooth
function defined by
\ben
H(\alpha) = \begin{cases}
0  &\text{ if } \alpha \le -1 \\
\frac12 + \frac{15}{16}(\alpha-\frac23 \alpha^3 + \frac15 \alpha^5)&\text{ if } -1 < \alpha < 1 \\
1  &\text{ otherwise.}
\end{cases}
\enn
By replacing the hinge loss with its smooth approximation $K_h(\alpha)$, a smooth SVM is obtained,
and a linear-type estimator is constructed for the smooth SVM in a distributed setting in \cite{Wang2019}.

Although there have already been several smooth loss functions to replace the Hinge loss in practice,
they may not be ideal choices due to the fact that they are either not approximations to the Hinge loss or at most
twice differentiable, such as the squared Hinge loss and its approximations as well as $\ell_\gamma$ and $K_h$.
In this section, we propose two (infinitely differentiable) smooth Hinge loss functions which  
overcome the above weakness and are given by
\begin{align}\label{SmoothHL}
\begin{split}
&\psi_G(\alpha;\sigma)=\Phi(v)(1-\alpha)+\phi(v)\sigma,\\
&\psi_{M}(\alpha;\sigma)=\Phi_{M}(v)(1-\alpha)+\phi_{M}(v)\sigma,
\end{split}
\end{align}
where $\sigma > 0$ is a given parameter and $v=({1-\alpha})/{\sigma}$.
Here, $\Phi(\cdot)$ and $\phi(\cdot)$ are the cumulative distribution function (CDF) and probability
density function (PDF) of the standard normal distribution, respectively,
$\Phi_{M}(v)=(1+{v}/{\sqrt{1+v^2}})/2,$ $\phi_{M}(v)={1}/(2{\sqrt{1+v^2}})$.
The following theorem gives the approximation property of $\psi_G$ and $\psi_{M}$.

\begin{theorem}\label{UniSmoothHL}
$\psi_G(\alpha;\sigma)$ and $\psi_{M}(\alpha;\sigma)$ satisfy the estimates:\\
1) $0\le\psi_G(\alpha;\sigma)-\max\{0,1-\alpha\}\le{\sigma}/{\sqrt{2\pi}}$;\\
2) $0\le\psi_{M}(\alpha;\sigma)-\max\{0,1-\alpha\}\le{\sigma}/{2}$.\\
Thus $\psi_G(\alpha;\sigma)$ and $\psi_{M}(\alpha;\sigma)$ converge to the Hinge loss uniformly in $\alpha$
as $\sigma$ tends to $0.$
\end{theorem}

\begin{proof}
Let $\ell(\alpha)=\max\{0,1-\alpha\}$ and $v=({1-\alpha})/{\sigma}$.
Taking the derivative of $\psi_G(\alpha;\sigma)$, we have
\ben
\psi_G'(\alpha;\sigma) = -(\Phi'(v)v + \phi'(v))- \Phi(v).
\enn
By the definition of $\phi(v)$ and $\Phi(v)$ we have
\ben
\Phi'(v)=\phi(v)=\frac{1}{\sqrt{2 \pi}}e^{-\frac12v^2}.
\enn
It then follows that
\be\label{EquationCG}
\Phi'(v)v+\phi'(v)=0.
\en
Thus we have $\psi_G'(\alpha;\sigma)=-\Phi(v)\le 0$, which means that $\psi_G$ is a monotonically decreasing
function of $\alpha$. For $\alpha\ge1$, we have $\ell(\alpha)=0$, so
\ben
0=\lim_{\alpha\to\infty}\psi_G(\alpha;\sigma)\le\psi_G(\alpha;\sigma)-\ell(\alpha)
\le\psi_G(1;\sigma).
\enn
For $\alpha\le 1$, $\ell(\alpha)=1-\alpha$ and $(\psi_G(\alpha;\sigma)-\ell(\alpha))' =1-\Phi(v)\ge 0$,
implying that
\ben
0=\lim_{\alpha\to-\infty}[\psi_G(\alpha;\sigma)-\ell(\alpha)]\le\psi_G(\alpha;\sigma)-\ell(\alpha)
\le\psi_G(1;\sigma).
\enn
Thus,
\ben
0\le\psi_G(\alpha;\sigma)-\max\{0,1-\alpha\}\le\psi_G(1;\sigma)=\frac{\sigma}{\sqrt{2\pi}}.
\enn

For $\psi_{M}(\alpha;\sigma)$ we have
\ben
\psi_{M}'(\alpha;\sigma)=-(\Phi_M'(v)v +\phi_M'(v))-\Phi_M(v).
\enn
By the definition of $\Phi_M(v)$ and $\phi_M(v)$ it is easy to see that
\begin{align}\label{EquationCM}
\Phi_M'(v)v + \phi_M'(v)=0.
\end{align}
Thus, $\psi_M'(\alpha;\sigma)=-\Phi_M(v)\le 0$, implying that $\psi_M$ is monotonically decreasing.
Similarly as for $\psi_G$, we can easily prove that
\ben
0\le\psi_M(\alpha;\sigma)-\max\{0,1-\alpha\}\le\psi_M(1;\sigma)={\sigma}/{2}.
\enn
The proof is thus complete.
\end{proof}

\section{A General Smooth Convex Loss}\label{sec3}

Motivated by \eqref{EquationCG} and \eqref{EquationCM} we propose a general smooth convex loss function
as stated in the following theorem.

\begin{theorem}\label{GCL}
Given $\theta\in\R$ and $\sigma>0$, define $\psi(\alpha)=\Phi_c(v)(\theta-\alpha)+\phi_c(v)\sigma$,
where $v=(\theta-\alpha)/\sigma$, $\Phi_c(v)$ and $\phi_c(v)$ are differentiable and satisfy
that $\Phi_c'(v)v+\phi_c'(v)=0$ and $\Phi_c'(v)\ge 0$. Then we have\\
1) $\psi(\alpha)$ is twice differentiable with $\psi'(\alpha)=-\Phi_c(v)$ and $\psi''(\alpha)=\Phi_c'(v)/\sigma$;\\
2) $\psi(\alpha)$ is convex;\\
3) $\psi(\alpha)$ is $\gamma$-strongly convex if $\Phi_c'(v)\ge\gamma\sigma$ for all $v\in\R$;\\
4) $\psi(\alpha)$ is $\mu$-smooth convex if $\Phi_c'(v)\le\mu\sigma$ for all $v\in\R$;\\
5) $\psi(\alpha)$ is classification-calibrated for binary classification if $\Phi_c({\theta}/{\sigma})>0$;\\
6) the conjugate of $\psi(\alpha)$ is $\psi^{\star}(\beta)=\beta\theta-\phi_c(\Phi_c^{-1}(-\beta))\sigma,$
$\beta\in-R(\Phi_c)$, where $\Phi_c^{-1}$ is the inverse function of $\Phi_c$ and $R(\Phi_c)$ is the range of $\Phi_c$.
\end{theorem}

\begin{proof}
1) It is easy to obtain that
\ben
\psi'(\alpha)&=&-(\Phi_c'(v)v+\phi_c'(v))-\Phi_c(v)=-\Phi_c(v),\\
\psi''(\alpha)&=&\Phi_c'(v)/\sigma.
\enn
2) Since $\psi''(\alpha)=\Phi_c'(v)/\sigma\ge 0$, $\psi$ is convex.\\
3) If $\Phi_c'(v)\ge\gamma\sigma$ for all $v\in\R$, then $\psi''(\alpha)\ge\gamma$. Thus
$\psi(\alpha)$ is $\gamma$-strongly convex, that is,
\ben
\psi(\alpha)-\psi(\beta)\ge\psi'(\beta)(\alpha-\beta)+\frac{\gamma}{2}|\alpha-\beta|^2,\;\;
\forall\alpha,\beta\in\R.
\enn
4) If $\Phi_c'(v)\le\mu\sigma$, then $\psi''(\alpha)\le\mu$, so the convex function $\psi(\alpha)$
is $\mu$-smooth, that is,
\ben
|\psi'(\alpha)-\psi'(\beta)|\le\mu|\alpha-\beta|,\;\;\forall\alpha,\beta\in\R.
\enn
5) Since $\psi'(0)=-\Phi_c({\theta}/{\sigma})<0$, by Theorem \ref{Bartlett} the convex function $\psi(\alpha)$
is classification-calibrated.\\
6) The conjugate function of $\psi$ is
\ben
\psi^{\star}(\beta) &=&\sup_{\alpha\in\R}(\beta\alpha-\psi(\alpha)) \\
&=& \sup_{\alpha\in\R}(\beta\alpha-\Phi_c(v)(\theta-\alpha) - \phi_c(v)\sigma) \\
&=& \beta \theta - \inf_{v\in\R}( \beta v + \Phi_c(v)v + \phi_c(v)) \sigma.
\enn
Let $L(v)=\beta v +\Phi_c(v)v+\phi_c(v)$ and $v^{\star}=\arg\inf_{v\in\R}L(v)$.
Then $L'(v)=\beta +\Phi_c(v)$.
Thus, $L(v)$ reaches its minimum at $v^{\star}=\Phi_c^{-1}(-\beta)$, and so we have
\ben
\psi^{\star}(\beta)&=&\beta\theta-(\beta v^{\star}+\Phi_c(v^{\star})v^{\star}+\phi_c(v^{\star})) \\
&=&\beta\theta-\phi_c(v^{\star})=\beta\theta-\phi_c(\Phi_c^{-1}(-\beta)).
\enn
\end{proof}

Note that if $\Phi_c$ and $\phi_c$ satisfy the conditions in Theorem \ref{GCL} and $\Phi_c(v) \ge 0$,
suppose $v=v(\alpha)$ is convex and differentiable, then $\psi_v(\alpha)=\Phi_c(v)v+\phi_c(v)$ is also convex.
The general smooth convex loss function $\psi(\alpha)$ includes many surrogate loss functions mostly used
in binary classification as special cases, as shown below.
Figure \ref{figbinary} below presents several surrogate convex loss functions.

{\bf Example 1: Least Square Loss.} Let $\Phi_c(v)=\sigma v$ and $\phi_c(v)=-\sigma v^2/2$.
Let $\theta>0$. Then
\ben
\psi(\alpha)=\Phi_c(v)(\theta-\alpha)+\phi_c(v)\sigma =(\theta-\alpha)^2/2,
\enn
which is the least square loss. The parameter $\theta$ satisfies $\Phi_c({\theta}/{\sigma})=\theta>0$.
The conditions in Theorem \ref{GCL} are easy to verify, and it is easy to obtain that
\ben
\psi'(\alpha)&=&-\Phi_c(v)=\alpha-\theta,\;\;\psi''(\alpha)={\Phi_c'(v)}/{\sigma}=1,\\
\psi^{\star}(\beta)&=&\beta\theta-\phi_c(\Phi_c^{-1}(-\beta))
=\beta\theta +\sigma(-{\beta}/{\sigma})^2/2 \\
&=& \beta\theta+{\beta^2}/({2\sigma}),\;\; \beta\in\R.
\enn

{\bf Example 2: Smooth Hinge Loss $\psi_G$.} Let $\Phi_c(v)=\Phi(v)$ and $\phi_c(v)=\phi(v)$
with $\Phi$ and $\phi$ given in Section \ref{sec2}. Then
\ben
\psi(\alpha)=\Phi(v)(\theta-\alpha)+\phi(v)\sigma.
\enn
By \eqref{EquationCG}, $\Phi_c'(v)v+\phi_c'(v)=0$, so $\psi'(\alpha)=-\Phi(v)\le0$.
Then $\psi(\alpha)\ge\lim_{\alpha\to\infty}\psi(\alpha)=0$. Setting $\theta=1$ gives the smooth Hinge
loss $\psi_{G}$. Further, $\Phi_c'(v)=\phi(v)\ge 0$, $\forall v\in\R$,
$\Phi_c({\theta}/{\sigma})=\Phi({\theta}/{\sigma})>0$, $\forall\theta\in\R$, and
$\psi(\alpha)$ tends to the Hinge loss as $\sigma\to0$.
It is also easy to see that
\ben
\psi''(\alpha)&=&{\Phi_c'(v)}/{\sigma}={\phi(v)}/{\sigma}, \\
\psi^{\star}(\beta)&=&\beta\theta-\phi_c(\Phi_c^{-1}(-\beta))=\beta\theta-\phi(\Phi^{-1}(-\beta))
\enn
with $\beta\in (-1,0)$.

{\bf Example 3: Smooth Hinge Loss $\psi_{M}$.} Let $\Phi_c(v) = \Phi_{M}(v)$ and
$\phi_c(v) = \phi_{M}(v)$. Then it follows that
\ben
\psi(\alpha)
&=& \Phi_c(v)(\theta-\alpha)+\phi_c(v)\sigma \\
&=& \frac12 (\theta-\alpha) + \frac12 \sqrt{(\theta-\alpha)^2 + \sigma^2}.
\enn
By \eqref{EquationCM}, $\Phi_c'(v)v+\phi_c'(v)=0$, so $\psi'(\alpha)=-\Phi_M(v)\le0$.
Then $\psi(\alpha)\ge\lim_{\alpha\to\infty}\psi(\alpha)=0$.
Moreover, $\Phi_c(\frac{\theta}{\sigma})\ge0$, $\forall\theta\in\R$,
$\Phi_M'(v)=(1+v^2)^{-3/2}/2>0$, $\forall v\in\R$, and $\psi(\alpha)$ tends to the Hinge loss as $\sigma\to0$.
Setting $\theta=1$ gives the smooth Hinge loss $\psi_{M}$.
It is easy to get
\ben
\psi''(\alpha)={\Phi_M'(v)}/{\sigma}=(1+v^2)^{-3/2}/(2\sigma).
\enn
In addition, we have
\ben
\psi^{\star}(\beta)=\beta\theta-\phi_c(\Phi_c^{-1}(-\beta))=\beta\theta-\frac12\sqrt{1-(2\beta+1)^2}
\enn
with $\beta\in (-1,0)$.

{\bf Example 4: Exponential Loss.} Let $\Phi_c(v)=e^{v}$, $\phi_c(v)=(1-v)e^v$.
Then $\psi(\alpha)=\Phi_c(v)(\theta-\alpha)+\phi_c(v)\sigma=\sigma e^v\ge0.$
Further, $\Phi_c'(v)=e^v\ge0$ and $\Phi_c({\theta}/{\sigma})=e^{\theta/\sigma}>0$
for all $\theta\in\R$. Moreover, $\Phi_c'(v)v+\phi_c'(v)=e^vv +(1-v)e^v-e^v=0$.
Letting $\theta=0$ and $\sigma =1$ gives the exponential loss. It is easy to get that
\ben
\psi'(\alpha)=-\Phi_c(v)=-e^{v},\;\;\psi''(\alpha)={\Phi_c'(v)}/{\sigma}={e^v}/{\sigma},
\enn
and
\ben
\psi^{\star}(\beta)&=&\beta\theta-\phi_c(\Phi_c^{-1}(-\beta))=\beta\theta+\beta(1-\ln(-\beta))\\
&=&\beta(1+\theta-\ln(-\beta)),\;\beta \in (-\infty,0).
\enn

{\bf Example 5: Logistic Loss.} Let $\Phi_c(v)={e^v}/({1+e^v})$, $\phi_c(v)=\ln(1+e^v)-v{e^v}/({1+e^v})$.
Then
\ben
\psi(\alpha)=\Phi_c(v)(\theta-\alpha)+\phi_c(v)\sigma=\sigma\ln(1+e^v)\ge0.
\enn
It is easy to get that $\Phi_c'(v)=e^v(1+e^v)^{-2}$, $\Phi_c({\theta}/{\sigma})>0$
for all $\theta\in\R$, and $\Phi_c'(v)v+\phi_c'(v)=0$.
Letting $\theta=0$ and $\sigma =1$ gives the logistic loss.
Further, we have
\ben
\psi'(\alpha)=-\frac{e^v}{1+e^v},\;\;\psi''(\alpha)=\frac{\Phi_c'(v)}{\sigma}=\frac{1}{\sigma}e^v(1+e^v)^{-2}.
\enn
For $\beta\in(-1,0)$, $\Phi_c^{-1}(-\beta)=\ln(-{\beta}/({1+\beta}))$, so
\ben
\psi^{\star}(\beta)&=&\beta\theta-\phi_c(\Phi_c^{-1}(-\beta)) \\
&=&\beta(\theta-\ln(-\beta))+(1+\beta)\ln(1+\beta).
\enn

{\bf Example 6: Smooth Absolute Loss.} Let $\Phi_c(v)=\arctan(v)$, $\phi_c(v)=-\frac12\ln(1+v^2)$.
Then it follows that
\ben
\psi(\alpha)=\arctan(v)(\theta-\alpha)-\frac{\sigma}2\ln(1+v^2),
\enn
$\Phi_c'(v)v+\phi_c'(v)=0$ and $\psi'(\alpha)=-\arctan((\theta-\alpha)/\sigma)$,
so $\psi'(\alpha)\le0$ for $\alpha\le\theta$ and $\psi'(\alpha)\ge0$ for $\alpha\ge\theta$.
Thus, $\psi(\alpha)\ge\lim_{\alpha\to\theta}\psi(\alpha)=0$.
To make the condition $\Phi_c({\theta}/{\sigma})=\arctan({\theta}/{\sigma})>0$
to hold, we must have $\theta>0$.
Further, we have $\Phi_c'(v)={1}/({1+v^2})\ge0$. It is also easy to derive that
$\lim_{\sigma\to0}\psi(\alpha)=(\pi/2)|\theta-\alpha|$, making us to call $\psi(\alpha)$
the smooth absolute loss function.
A direct calculation gives
\ben
\psi''(\alpha)=\frac{\Phi_c'(v)}{\sigma}=\frac{1}{\sigma(1+v^2)}.
\enn
Finally, for $\beta\in(-\frac{\pi}{2},\frac{\pi}{2})$ we have $\Phi_c^{-1}(-\beta)=\tan(-\beta)$ and
\ben
\psi^{\star}(\beta)=\beta\theta-\phi_c(\Phi_c^{-1}(-\beta))=\beta\theta+\frac12\ln(1+\tan^2(-\beta)).
\enn

{\bf Example 7: Smooth ReLU.} ReLU is a famous non-smooth activation function in
deep neural networks (DNN), which is defined as $\psi_{ReLU}(\alpha)=\max(0,\alpha)$.
Define the smooth ReLU (sReLU) function as
\ben
\psi_{sReLU}(\alpha;\sigma)=\Phi({\alpha}/{\sigma})\alpha+\phi({\alpha}/{\sigma})\sigma
\enn
with $\Phi$ and $\phi$ given in Section \ref{sec2}. Then $\psi_{sReLU}(\alpha;\sigma)=\psi(-\alpha)$,
where $\psi$ is defined as in Example 2 with $\theta=0$. By Example 2 we know that
$\psi_{sReLU}(\alpha;\sigma)$ uniformly converge to ReLU as $\sigma$ goes to $0.$

By Theorem \ref{GCL} we have the following remarks.

1) Any smooth convex function $\psi(\alpha)$ can be rewritten in the form
$\psi(\alpha)=\Phi_c(v)(\theta-\alpha)+\phi_c(v)\sigma$, where $\Phi_c(v)=-\psi'(\theta-\sigma v)$
and $\phi_c(v)=\psi(\theta-\sigma v)-\Phi_c(v)v$.

2) Given a monotonically increasing, differentiable function $\Phi_c(v)$, we are able to construct
a convex, smooth surrogate loss which is classification-calibrated for binary classification.
Moreover, there is no need to know the explicit expression of the loss function when learning
with gradient-based methods.

3) There is a great interest to develop a smoothing technique to approximate a non-smooth
convex function \cite{Nesterov2005,Shi2014}.
For example, \cite{Nesterov2005} improved the traditional bounds on the number of iterations
of the gradient methods based on a special smoothing technique in non-smooth convex minimization
problems. Theorem \ref{GCL} provides a new smoothing technique by searching a monotonically
increasing and differentiable function $\Phi_c(v)$ which approximates the sub-gradient of
the non-smooth convex function.

\begin{figure}
\centering
\includegraphics[width=0.49\textwidth]{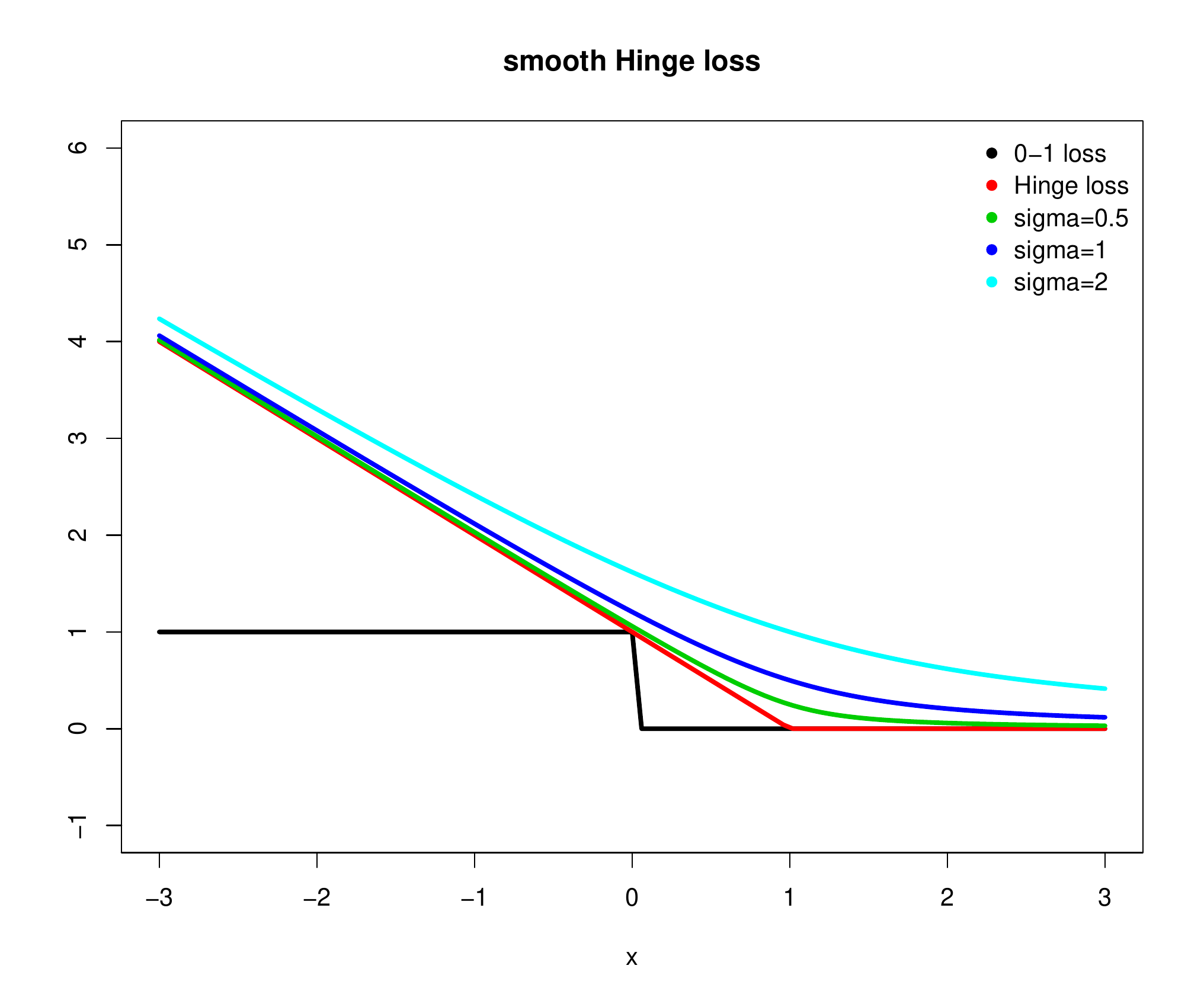}
\includegraphics[width=0.49\textwidth]{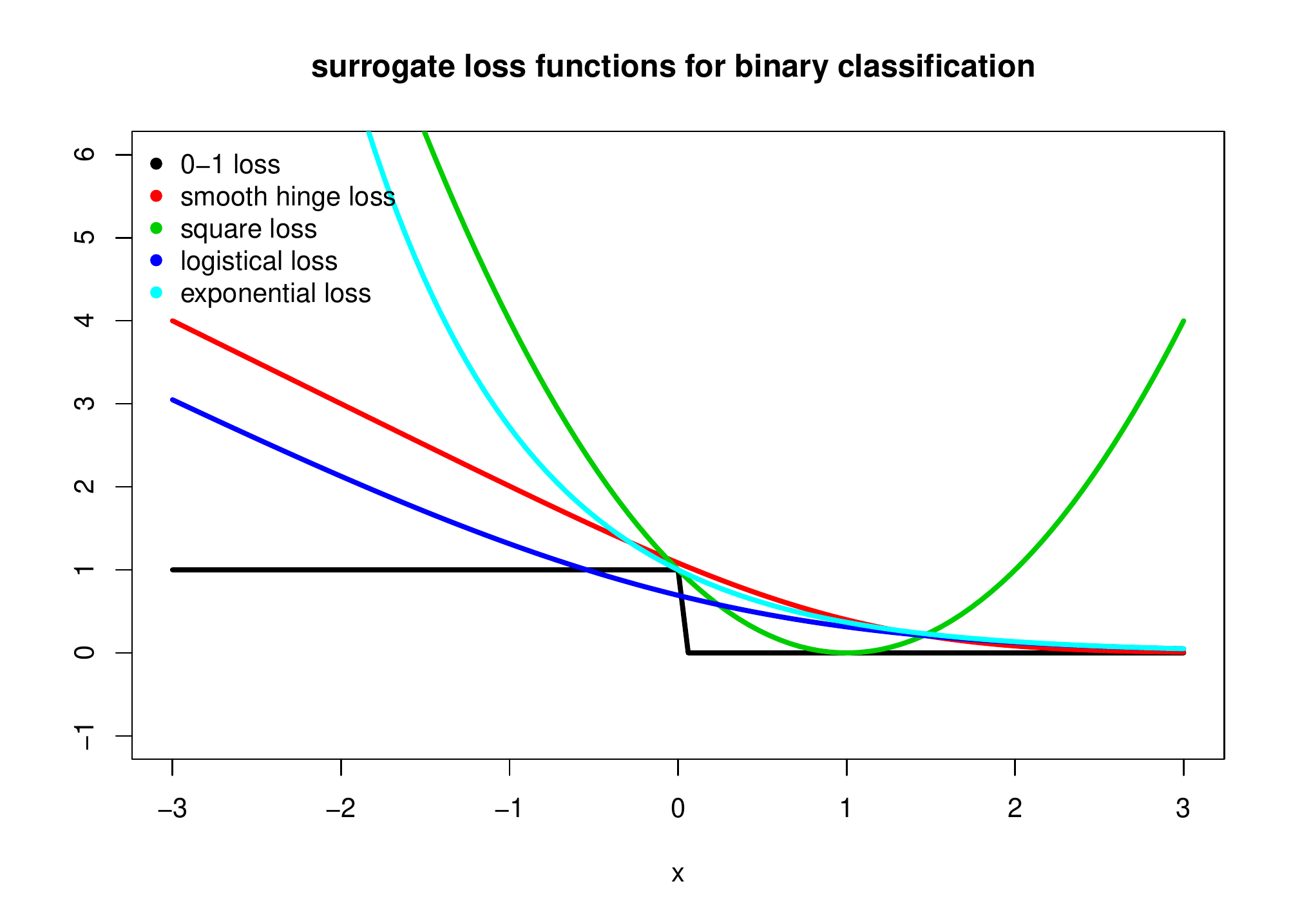}
\caption{\small Several surrogate convex loss functions for the $0/1$ loss used in binary classification.
The left figure presents the smooth Hinge loss $\psi_M$ with different parameter $\sigma$, where,
as $\sigma$ goes to $0$, the smooth Hinge loss gets close to the Hinge loss.
The right figure presents several popular surrogate loss functions.}
\label{figbinary}
\end{figure}

\section{Algorithms}\label{sec4}

There are many algorithms for SVMs. In the early time, the decomposition methods such as SMO
and $\text{SVM}^{\text{light}}$, which overcome the memory requirement of the quadratic optimization methods,
have been proposed to solve L1-SVM in its dual form \cite{Platt1998,Joachims1999}
(see \cite{Zhang2018} for convergence analysis of the decomposition methods including SMO
and $\text{SVM}^{\text{light}}$).
Later, several convex optimization methods have been introduced, such as the gradient-based methods \cite{Zhang2004},
the bundle method \cite{Teo2010,Franc2008}, the coordinate descent method \cite{Zhang2001,Chang2008},
the dual method \cite{Hsieh2008,Shalev-Shwartz2013,Zheng2013} and online learning
methods \cite{Shalev2011,WangDi2013,Ding2018}.
Based on the generalized Hessian matrix of a convex function with locally Lipschitz gradient
proposed in \cite{Mangasarian2002}, several second-order methods have been applied to
solve L2-SVM \cite{Keerthi2005,Lin2008}.

In this paper, we focus on the following smooth support vector machine
\be\label{SSVM}
\arg\min_{w\in\R^p} L(w):=\half\lambda\|w\|^2+\frac1n\sum_{i=1}^n\psi(y_iw^Tx_i;\sigma),
\en
where $\psi$ is the smooth Hinge loss $\psi_G$ or $\psi_M$.
We will analyze the first and second-order convex algorithms for the above SSVM.

\subsection{First-Order Algorithms}

The gradient descent (GD) method has taken the stage as the primary workhorse for convex optimization problems.
It iteratively approaches the optimal solution.
In each iteration, the standard full gradient descent (FGD) method takes the negative gradient
as the descent direction, and the classifier updates as follows
\be\label{GD}
w_{t+1}=w_t-\eta_t\left[\lambda w_t-\frac1n\sum_{i=1}^n\Phi_c(v_i^t)y_ix_i\right],
\en
where $\eta_t$ is a predefined step size and $v_i^t=({1-y_iw_t^Tx_i})/{\sigma}$.
The FGD method has a $O(1/t)$ convergence rate under standard assumptions \cite{Nesterov2004}.
Nesterov proposed a famous accelerated gradient (AGD) method in 1982 (see \cite{Nesterov2004}).
The AGD method achieves the optimal convergence rate of $O(1/t^2)$ for smooth objective functions.

However, for large-scale problems it is computationally very expensive to compute the full gradient
in each iteration. To accelerate the learning procedure, the linear classifier can be updated with a
stochastic gradient (SG) step instead of a full gradient (FG) step, such as Pegasos \cite{Shalev2011}.
Precisely, at the $t$-th iteration, $w_t$ is updated based on a randomly chosen example $(x_{i_t},y_{i_t})$:
\be\label{SGD}
w_{t+1} = w_t -\eta_t\left[\lambda w_t-\Phi_c(v_{i_t}^t)y_{i_t}x_{i_t}\right],
\en
where $\eta_t$ is a predefined step size which is required to satisfy that $\sum_{t=1}^{\infty}\eta_t=\infty$
and $\sum_{t=1}^{\infty}\eta_t^2<\infty$ for convergence.
$\eta_t\Phi_c(v_{i_t}^t)$ can be seen as the learning rate of the chosen example $(x_{i_t},y_{i_t})$.
Noting that $\eta_t$ is independent of $(x_{i_t},y_{i_t})$ and $\Phi_c(v_{i_t}^t)$ depends on the margin of
the chosen example $(x_{i_t},y_{i_t})$, we call $\eta_t$ the exogenous learning rate and $\Phi_c(v_{i_t}^t)$
the endogenous learning rate.

An advantage of the above stochastic gradient descent (SGD) algorithm is that its iteration cost
is independent of the number of training examples. This property makes it suitable for large-scale problems.
In the SGD algorithm, the full gradient $-(1/n)\sum_{i=1}^n\Phi_c(v_i^t)y_ix_i$ is replaced by
a stochastic gradient $-\Phi_c(v_{i_t}^t)y_{i_t}x_{i_t}$ of a randomly chosen example in the $t$-th iteration.
Though both gradients are equivalent in expectation, they are rarely the same. Thus, it is not a natural
to use the norm of the stochastic gradient, $\|\Phi_c(v_{i_t}^t)y_{i_t}x_{i_t}\|$, as a stopping criterion.
The discrepancy between the FG and the SG also has a negative effect on the convergence rate of the SGD method.
Since there is no guarantee that $\|\Phi_c(v_{i_t}^t)y_{i_t}x_{i_t}\|$ will approach to zero,
we need to employ a monotonically decreasing step size series $\{\eta_t\}_{t=0}^\infty$ with $\eta_t\to0$
for convergence. The small step size leads to a sub-linear convergence of the SGD method even
for the strongly convex objective function.

Many first-order algorithms have been proposed by combining the low computation cost of SGD and the faster
convergence of the FGD method to provide variance reduction \cite{Defazio2014a,wangc2013}.
These novel algorithms are known as stochastic variance reduction gradient (SVRG) algorithms,
such as SAG \cite{Schmidt2017}, SAGA \cite{Defazio2014a}, SDCA \cite{Shalev-Shwartz2013},
Finito \cite{Defazio2014b}, and SVRG \cite{wangc2013}. Their convergence analysis can be found in
the corresponding references.

\begin{algorithm}
\renewcommand{\algorithmicrequire}{\textbf{Input:}}
\renewcommand{\algorithmicensure}{\textbf{Output:}}
\caption{The Framework of TRON Algorithm}\label{TRON}
\begin{algorithmic}[1]
\REQUIRE Given $w_0=0$, training data $X$, label vector $Y$ and parameter $\lambda,\sigma$,
$\Delta_0=\|\nabla L(w_0)\|$\\
    \FOR {$t= 1,2,...$}
    \STATE Calculate $g_t=\nabla L(w_t)$ and \mbox{stop if}\;$\|\nabla L(w_t)\|\le 0.0005$\\
    \STATE Calculate the diagonal matrix $D$ with respect to the Hessian matrix $\nabla^2L(w_t)$\\
    \STATE Find an approximate solution $s_t$ of the trust region sub-problem with
           the conjugate gradient method in Algorithm \ref{CGalgorithm}\\
    \STATE $\rho_t=\frac{L(w_t+s_t)-L(w_t)}{q_t(s_t)}$, where
    \begin{align*}
    &q_t(s_t)=\nabla^T L(w_t)s_t+\frac12 s_t^T(\lambda E+\frac1n X^TDX)s_t \\
    &=\nabla^T L(w_t)s_t+\frac12\lambda\|s_t\|^2+\frac{1}{2n}(Xs_t)^TD(Xs_t)
    \end{align*}
    \STATE Update the classifier $w_{t+1}$:
    \ben
     w_{t+1}=\begin{cases}
     w_t+s_t &\qquad\mbox{if}\;\rho_t>\eta_0 \\
     w_t &\qquad\mbox{if}\;\rho_t\le\eta_0
    \end{cases}
    \enn
    \STATE Update $\Delta_{t+1}$ according to the rule:
    \ben
    &&\Delta_{t+1}\in[\delta_1\min(\|s_t\|,\Delta_t),\delta_2\Delta_t],\;\mbox{if}\;\rho\le\eta_1; \\
    &&\Delta_{t+1}\in[\delta_1\Delta_t,\delta_2\Delta_t],\;\mbox{if}\;\eta_1<\rho<\eta_2; \\
    &&\Delta_{t+1}\in[\Delta_t,\delta_3\Delta_t],\;\mbox{if}\;\rho\ge\eta_2.
    \enn
    \ENDFOR
  \end{algorithmic}
\end{algorithm}

\begin{algorithm}
  \renewcommand{\algorithmicrequire}{\textbf{Input:}}
  \renewcommand{\algorithmicensure}{\textbf{Output:}}
  \caption{Conjugate Gradient Algorithm}\label{CGalgorithm}
  \begin{algorithmic}[1]
    \REQUIRE $s_t^1\gets 0$, $r_t^1\gets(\nabla^2 L(w_t))s_t^1-\nabla L(w_t)=-\nabla L(w_t)$,
    $p_t^1=-\nabla L(w_t)$, $\Delta_t$
    \FOR {$k= 1,2,3,...$}
    \STATE If $\|r_t^k\|\le\xi_t\|g_t\|$, output $s_t=s_t^{k}$ and {\bf Terminate}.
    \STATE $\alpha_t^k=\frac{(r_t^k)^T r_t^k}{(p_t^k)^T\nabla^2 L(w_t)p_t^k}$\\
          $\;\;\;\;=\frac{(r_t^k)^T r_t^k}{\lambda(p_t^k)^Tp_t^k+\frac1n(Xp_t^k)^T(D(Xp_t^k))}$
    \STATE If $\|s_t^k+\alpha_t^k p_t^k\|\le\Delta_t$, set $s_t^{k+1}=s_t^k+\alpha_t^k p_t^k$
           and continue with Step 5.
           Otherwise, compute $\tau>0$ so that $\|s_t^k+\tau p_t^k\|=\Delta_t$.
           Set $s_t^{k+1}=s_t^k+\tau p_t^k$ and {\bf Terminate}. Compute
           $\tau=\frac{-(s_t^k)^Tp_t^k+\sqrt{((s_t^k)^Tp_t^k)^2+\|p_t^k\|^2(\Delta_t^2-\|s_t^k\|^2)}}{\|p_t^k\|^2}$
    \STATE $r_t^{k+1}=r_t^k-\alpha_t^k(\nabla^2 L(w_t) p_t^k)=r_t^k-\lambda\alpha_t^k p_t^k
           -\frac1n\alpha_t^k(X^T(D(X p_t^k)))$.
    \STATE $\beta=\frac{(r_t^{k+1})^Tr_t^{k+1}}{(r_t^k)^Tr_t^k}$
    \STATE $p_t^{k+1}=r_t^{k+1}+\beta p_t^k$
    \ENDFOR
  \end{algorithmic}
\end{algorithm}

\subsection{Second-Order Algorithms}

Recently, inexact Newton methods without computing the inverse Hessian matrix have been proposed to obtain a
superlinear convergence rate in machine learning, such as
LiSSA \cite{Agarwal2017} and TRON \cite{Lin1999,Lin2008}.
LiSSA constructs a natural estimator of the inverse Hessian matrix by using the Taylor expansion, while
TRON is a trust region Newton method introduced in \cite{Lin1999} to deal with general
bound-constrained optimization problems, which generates an approximate Newton direction by solving
a trust-region subproblem.
In TRON, the direction step should give as much reduction as the Cauchy step.
\cite{Lin2008} applied the TRON method to maximize the log-likelihood of the logistic regression model,
in which a conjugate gradient method was used to solve the trust-region subproblem approximately.
TRON was also extended to solve the L2-SVM model by introducing a general Hessian for convex objective
functions having a Lipschitz continuous gradient \cite{Mangasarian2002}.

We want to apply TRON to solve our SSVM problem \eqref{SSVM}. By Theorem \ref{UniSmoothHL} we have
\ben
\nabla^2\psi(y_iw^Tx_i)=x_i\left(\frac{\Phi_c'(v_i)}{\sigma}\right)x_i^T=d_i x_i x_i^T,
\enn
where $v_i=({1-y_iw^Tx_i})/{\sigma}$ and $d_i=\Phi_c'(v_i)/\sigma$.
Then the Hessian matrix of the total loss function $L(w)$ is given by
\ben
\nabla^2 L(w)=\lambda E+\frac1n\sum_{i=1}^n d_i x_i x_i^T=\lambda E+\frac1n X^TDX,
\enn
where $E$ is the $p\times p$ identity matrix, $D=diag(d_{1},\cdots,d_{n})$ is a diagonal matrix,
and $X$ is the input feature matrix with its each row representing an instance.

The Newton step is given as $\nabla^{-2}L(w)\nabla L(w)$ which requires a huge computation cost
for high dimensional machine learning problems. The trust-region method is to provide an approximate
Newton direction. For recent advances in the trust-region methods see \cite{Yuan2015}.
Suppose $w_t$ is the solution at the $t$-th iteration. Then the trust-region method generates a
direction step $s_t$ by solving the quadratic subproblem
\be\label{subproblem}
s_t=\arg\min_{s\in\R^p}\left<\nabla L(w_t),s\right>+\frac12s^TBs,\;\mbox{s.t.}\;\|s\|\le\Delta_t,
\en
where $\Delta_t$ is the trust region and $B$ is a positive semi-definite matrix.
Here, we choose the matrix $B$ to be the true Hessian $\nabla^2 L(w_t)$.
By solving \eqref{subproblem} with the conjugate-gradient (CG) method, the objective function is
an adequate approximation of the reduction of the total loss $L(w_t+s_t)-L(w_t)$.
Thus, if the trust region $\Delta_t$ is large enough so that $\|\nabla^{-2}L(w_t)\nabla L(w_t)\|\le\Delta_t$,
then we will get an inexact Newton step.
The framework of TRON for the SSVM \eqref{SSVM} is presented in Algorithm \ref{TRON}, and the CG method
for the subproblem \eqref{subproblem} is given in Algorithm \ref{CGalgorithm}.
In updating $\Delta_{t+1}$ from $\Delta_t$, we use the same parameters
$\eta_0,\eta_1,\eta_2,\delta_1,\delta_2,\delta_3$ as in \cite{Lin2008}.

Note that there is no matrix inversion or matrix-matrix multiplication but a matrix-vector multiplication
in Algorithm \ref{TRON}. The structure of the Hessian matrix $\nabla^2 L(w)$ makes it appropriate
for $\nabla^2 L(w)$ to be applied to the second-order algorithm, especially when the input feature
matrix $X$ is sparse. For any vector $s\in\R^p$, the Hessian matrix-vector multiplication is given as
\ben
\nabla^2 L(w)s=(\lambda E+\frac1n X^TDX)s=\lambda s+\frac1n X^T(D(Xs)),
\enn
so there is no need to compute and store the dense and high-dimensional Hessian matrix $\nabla^2 L(w)$.
We can sequentially calculate $s_1=Xs,\;s_2=Ds_1,\;s_3=X^Ts_2$, which involves only sparse matrix-vector
multiplication. Thus, both the computation and storage costs can be controlled even for very
high-dimensional problems. The following theorem establishes the convergence rate of the TRON algorithm
(Algorithm \ref{TRON}), which is similar to that of the TRON algorithm developed for
logistic regression in \cite{Lin2008}.

\begin{theorem}\label{thm4}
The sequence $w_t$ generated by Algorithm \ref{TRON} globally converges to the unique optimal solution $w^\star$
of the SSVM (\ref{SSVM}).
If $\xi_t<1$, then Algorithm \ref{TRON} has a Q-linear convergence, that is,
$\lim_{t\to\infty}\frac{\|w_{t+1}-w^\star\|}{\|w_t-w^\star\|}<1$.
If $\xi_t\to0$ as $t\to\infty$, then Algorithm \ref{TRON} is of Q-superlinear convergence, that is,
$\lim_{t\to\infty}\frac{\|w_{t+1}-w^\star\|}{\|w_t-w^\star\|}=0$.
Moreover, if $\xi_t=\kappa\|\nabla L(w_t)\|$ for a positive constant $\kappa$, then $w_t$ converges
quadratically to $w^\star$, that is, $\lim_{t\to\infty}\frac{\|w_{t+1}-w^\star\|}{\|w_t-w^\star\|^2}<1$.
\end{theorem}

According to the convex optimization theory, it is not surprising that the trust region Newton method 
achieves a quadratic convergence rate for a smooth convex objective function $L(w)$ 
if $\|\nabla^{2}L(w)\|\le\mu$ for some $\mu$.
This is also true for other second-order algorithms such as BFGS and LBFGS.
Therefore, we can also apply other second-order algorithms such as BFGS and LBFGS to solve
the SSVM (\ref{SSVM}) with similar convergence results as Theorem \ref{thm4}.
Further, the second-order stochastic optimization algorithm LiSSA proposed in \cite{Agarwal2017}
for logistical regression can also be extended to solve the SSVM (\ref{SSVM}).
Note that the TRON algorithm may be applied to L2-SVM with the help of the generalized Hessian matrix.
However, the squared Hinge loss in L2-SVM is not twice differentiable, so there is no guarantee to
achieve the quadratic convergence rate of TRON for L2-SVM.

\section{Experiment}\label{sec5}

In this section, we conduct two experiments to study the first-order and second-order algorithms for
our SSVM (\ref{SSVM}) with the smooth Hinge loss $\psi_{G}(\alpha;\sigma)$ or $\psi_{M}(\alpha;\sigma)$.

\subsection{Datasets}

In this paper, we focus on the linear SVM model.
We consider three data sets: NEWS20, RCV1 and REAL-SIM, which come from document classification and can be
download from LIBSVM website
\footnote{https://www.csie.ntu.edu.tw/\url{~}cjlin/libsvmtools/datasets/}.
Table \ref{dataset} lists the number of instances and features as well as the sparsity metric of these
datasets which is the proportion of non-zero elements in the input feature.
Details can be found in \cite{Lin2008}.
\begin{table}[htbp]
\caption{The three sparse datasets used in the experiments.}\label{dataset}
\vskip 0.15in
\begin{center}
\begin{small}
\begin{sc}
\begin{tabular}{cccc}
\toprule
dataset &  data size & feature & sparsity \\
\midrule
NEWS20 & 19996 &  1355191 & 0.034\% \\
RCV1 & 697641 & 47236 &  0.155\% \\
REAL-SIM & 72309 & 20958 &  0.245\% \\
\bottomrule
\end{tabular}
\end{sc}
\end{small}
\end{center}
\vskip -0.1in
\end{table}

\subsection{The Smooth Parameter Sensitivity Study}

We first study the sensitivity of the smooth parameter $\sigma$.
As discussed in Section \ref{sec2}, the smooth Hinge loss functions $\psi_{G}(\alpha;\sigma)$
and $\psi_{M}(\alpha;\sigma)$ approach to the Hinge loss if $\sigma$ tends to $0.$
If $\sigma$ is too small, then SSVM is almost equivalent L1-SVM.
If $\sigma$ is too large, then the margin $yw^Tx$ has little effect on the smooth loss, which makes it
impossible to learn a good classifier.
We set $\lambda= 10^{-5}$, $\xi_t=0.1$ and choose $\sigma$ to be $2^{-30}$, $2^{-25}$, $2^{-20}$, $2^{-15}$,
$2^{-10}$, $2^{-9}$, $\cdots$, $2^{5}$.
For each dataset, we randomly divide it into $5$ parts, choose $4$ parts for training, and the remaining one
for testing.
To reduce the variance in the results, for each data set, we generate $4$ independent $5$-part
partitions, leading to a total $20$ runs for each dataset.
Then the training accuracy is obtained by averaging over these $20$ runs.
We present the training accuracy of the three datasets NEWS20, REAL-SIM, and RCV in Figure \ref{figaccuracy}.

\begin{figure}
\centering
\includegraphics[width=0.49\textwidth]{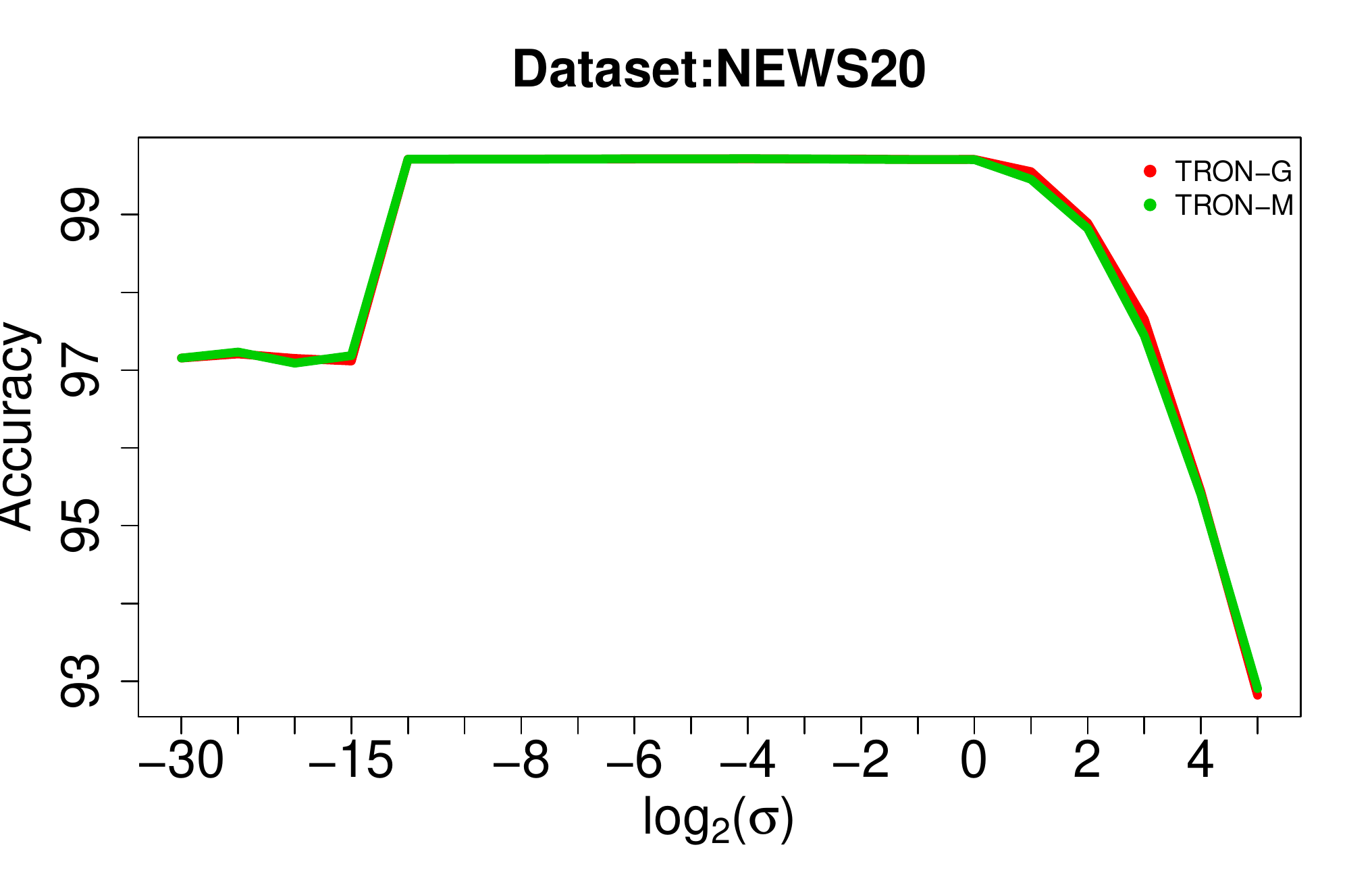}
\includegraphics[width=0.49\textwidth]{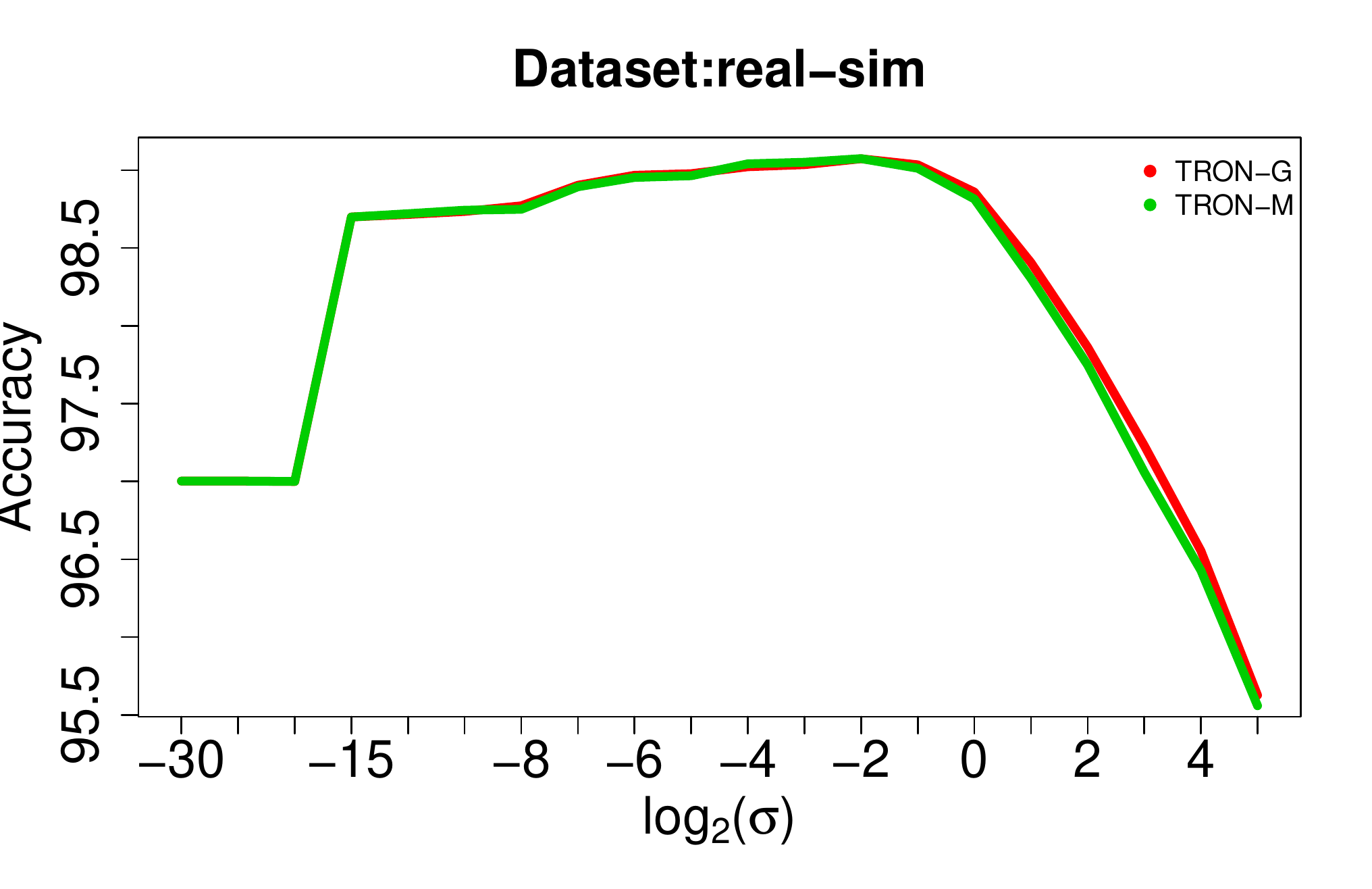}
\includegraphics[width=0.5\textwidth]{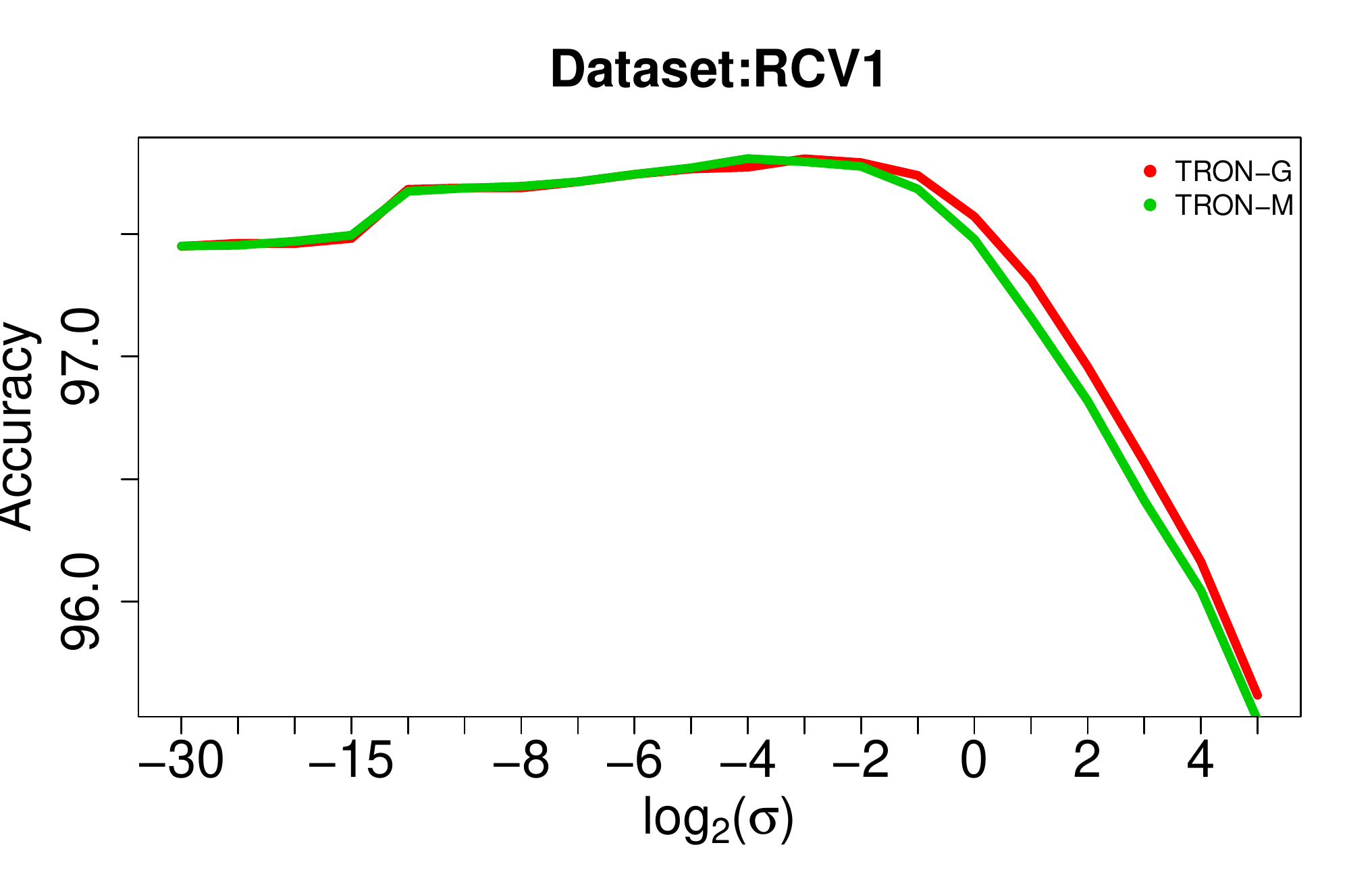}
\caption{The training accuracy of the two new models on the datasets with respect to the parameter $\sigma$:
NEWS20 (top left), REAL-SIM (top right) and RCV1 (bottom).
The vertical axis shows the average accuracy and the horizontal axis is $\log_2\sigma$.
The red and green lines represent $\psi_G$ and $\psi_M$, respectively.
}\label{figaccuracy}
\end{figure}

\begin{table*}[htbp]
\caption{The test accuracy (\%) of different models on the three datasets.}\label{Experiment}
\vskip 0.15in
\begin{center}
\begin{small}
\begin{sc}
\begin{tabular}{ccccc}
\hline
METHODS   & LOSSES                          &    NEWS20           &    REALSIM          &    RCV1 \\ \hline
TRON      & Logistic loss                   &  $ 95.99 \pm 0.21 $ &  $ 97.20 \pm 0.12 $ &  $ 97.23 \pm 0.04 $  \\
TRON      & squared Hinge loss $\ell_2$      &  $ 96.57 \pm 0.23 $ &  $ 97.40 \pm 0.12 $ &  $ 97.65 \pm 0.04 $  \\
TRON      & smooth Hinge loss $\psi_M$      &  $ 96.72 \pm 0.22 $ &  $ 97.36 \pm 0.13 $ &  $ 97.60 \pm 0.04 $  \\
TRON      & smooth Hinge loss $\psi_G$      &  $ 96.73 \pm 0.21 $ &  $ 97.41 \pm 0.12 $ &  $ 97.61 \pm 0.04 $  \\
PEGASOS   & Hinge loss $\ell_1$              &  $ 94.12 \pm 0.41 $ &  $ 96.42 \pm 0.16 $ &  $ 97.50 \pm 0.05 $  \\
SIFS      & smooth Hinge loss $\ell_\gamma$ &  $ 94.20 \pm 0.40 $ &  $ 96.77 \pm 0.14 $ &  $ 97.52 \pm 0.04 $  \\
\hline
\end{tabular}
\end{sc}
\end{small}
\end{center}
\vskip -0.1in
\end{table*}

\begin{table*}[htbp]
\caption{The running time(s) of different models on the three datasets.}\label{runningtime}
\vskip 0.15in
\begin{center}
\begin{small}
\begin{sc}
\begin{tabular}{ccccc}
\hline
METHODS   & LOSSES                          & NEWS20    & REALSIM  &    RCV1   \\ \hline
TRON      & Logistic loss                   &   1.82    &  0.46    &  7.43   \\
TRON      & squared Hinge loss $\ell_2$      &  10.23    &  1.31    &  31.97   \\
TRON      & smooth Hinge loss $\psi_M$      &  16.43    &  2.01    &  24.25   \\
TRON      & smooth Hinge loss $\psi_G$      &  19.95    &  2.45    &  19.63   \\
PEGASOS   & Hinge loss $\ell_1$              &  683.46   &  86.95   &  2486.50 \\
SIFS      & smooth Hinge loss $\ell_\gamma$ &  238.19   &  39.06   &  430.18 \\
\hline
\end{tabular}
\end{sc}
\end{small}
\end{center}
\vskip -0.1in
\end{table*}

\subsection{Comparison with Other Algorithms}

We now compare our SSVMs with several state-of-the-art binary classification models:
Pegasos for L1-SVM, TRON for L2-SVM, and TRON for logistic regression \cite{Lin2008,Shalev2011}.
These experiments were conducted with MATLAB on a workstation with 8GB memory and 3.60GHZ CPU.
We also compare our method with the new method SIFS proposed in 2019 \citep{Hong2019}
(the code of this method is obtained from https://github.com/jiewangustc/SIFS).
We use $\sigma=2^{-6},2^{-1},2^{-3}$ for the datasets NEWS20, REAL-SIM and RCV1, respectively.
The averaged testing accuracy is presented in Table \ref{Experiment}.
The results show that the second-order method is better than the first-order methods in the large-scale sparse
problems and our new model is effective for binary classification problems.
By using the same stopping criterion $\|\nabla L(w_t)\|\le 0.001$ for the same optimization algorithm TRON,
we compare several different surrogate loss functions: the logistic loss, the squared Hinge loss, and
the smooth Hinge losses. The logistic loss has minimum iteration steps, and the smooth Hinge losses
achieve the highest accuracy.

\section{Conclusions}\label{sec6}

In this paper, we proposed two smooth Hinge loss functions $\psi_G$ and $\psi_M$ to build
two smooth support vector machines for binary classification problems.
We have also discussed several modern first-order and second-order convex algorithms which can be applied 
to solve our SSVMs effectively. The second-order algorithms can achieve a quadratic convergence rate for 
the SSVMs, which is different from the traditional SVMs.
For our SSVMs and several state-of-the-art binary classification models, experiments have also been
carried out on three real-world data sets from document classification, and the experimental results
illustrated that our SSVMs are useful for binary classification problems.
Further, motivated by the smooth Hinge loss functions $\psi_G$ and $\psi_M$, we gave a general smooth convex
loss function which unifies several commonly-used convex loss functions in machine learning,
including the L1 regularization and the ReLU activation function in neural networks.
The unified framework provides a tool to approximate a non-smooth convex function with a smooth one
which can be used to build a smooth machine learning model to be solved with a faster convergent
optimization algorithm.

\bibliography{ssvm}
\end{document}